\documentclass[letterpaper, 10 pt, conference]{ieeeconf}  

\IEEEoverridecommandlockouts                              

\usepackage{geometry}
\newgeometry{left=16.9mm,right=16.9mm,bottom=15.2mm,top=20.1mm}
 
\usepackage[utf8]{inputenc}  
\usepackage{leftidx}  
\usepackage{mathtools}
\mathtoolsset{showonlyrefs,showmanualtags}
\usepackage{amssymb,amsthm,amsmath}
\usepackage{dsfont}
\usepackage{booktabs} 
\usepackage{color}
\usepackage{bbold} 
\usepackage[yyyymmdd,hhmmss]{datetime}
\usepackage{bm,upgreek}
\usepackage{tablefootnote}
\usepackage{nth}
\usepackage{pgfplots}
\usepackage{tikz}
\usetikzlibrary{shapes,snakes}
\usetikzlibrary{patterns}
\usepackage{multirow}
\usepackage{footnote}
\usetikzlibrary{hobby}

\usepackage{textcomp}
\usepackage{algorithm}
\usepackage{textcomp}
\usepackage{xcolor}
\usepackage{multicol}
\usepackage{algpseudocode}
\def\BibTeX{{\rm B\kern-.05em{\sc i\kern-.025em b}\kern-.08em
   T\kern-.1667em\lower.7ex\hbox{E}\kern-.125emX}}

\newtheorem{theorem}{Theorem}
\newtheorem{definition}{Definition}
\newtheorem{remark}{Remark}
\newtheorem{lemma}{Lemma}

\newtheorem{problem}{Problem}

\usepackage{babel}

\usepackage{tikz}
\usetikzlibrary{graphs,graphs.standard}

\tikzset{%
  every neuron/.style={
    circle,
    draw,
    minimum size=.7cm
  },
  neuron missing/.style={
    draw=none, 
    scale=3,
    text height=0.333cm,
    execute at begin node=\color{black}$\vdots$
  },
}

\definecolor{red}{RGB}{187,0,0}
\definecolor{blue}{RGB}{0, 0,180}
\definecolor{pink}{RGB}{203, 76, 178}

\algdef{SE}[DOWHILE]{Do}{doWhile}{\algorithmicdo}[1]{\algorithmicwhile\ #1}%
\algnewcommand\algorithmicinput{\textbf{Input:}}
\algnewcommand\algorithmicoutput{\textbf{Output:}}
\algnewcommand\Input{\item[\algorithmicinput]}
\algnewcommand\Output{\item[\algorithmicoutput]}

\usepackage{cite}
\title{\LARGE \bf
On Forward Kinematics of a 3SPR Parallel Manipulator
}

\author{Masoud Roudneshin, Kamran Ghaffari and Amir G. Aghdam
\thanks{Masoud Roudneshin and Amir G. Aghdam are with the department of Electrical and Computer Engineering, Concordia Univerity, Montreal, QC, Canada. Email: {\tt\small m\_roundne@encs.concordia.ca, amir.aghdam@concordia.ca}}
\thanks{Kamran Ghaffari is with Touché Technologies, Montreal, QC, Canada.
        Email: {\tt\small kamran.ghaffari@touche-technologies.com }}%
}

\begin{document}
\maketitle
\thispagestyle{empty}
\pagestyle{empty}

\begin{abstract}

In this paper, a new numerical method to solve the forward kinematics (FK) of a parallel manipulator with three-limb spherical-prismatic-revolute (3SPR) structure is presented. Unlike the existing numerical approaches that rely on computation of the manipulator's Jacobian matrix and its inverse at each iteration, the proposed algorithm requires much less computations to estimate the FK parameters. A cost function is introduced that measures the difference of the estimates from the actual FK values. At each iteration, the problem is decomposed into two steps. First, the estimates of the platform orientation from the heave estimates are obtained. Then, heave estimates are updated by moving in the gradient direction of the proposed cost function. To validate the performance of the proposed algorithm, it is compared against a Jacobian-based (JB) approach for a 3SPR parallel manipulator.   

\end{abstract}

\section{INTRODUCTION}
The forward kinematics (FK) problem has attracted researchers in various engineering fields and has fundamental applications in robotics. The problem is to find the manipulator’s workspace configuration with a given set of joint lengths or angles. The requirement to solve, in real-time, a set of nonlinear equations that contain products of trigonometric functions makes the FK problem a challenging one for parallel manipulators \cite{Goss,Opt,pol1,gough old,NR1,NR2}.

There is a vast body of recent literature on methods to solve FK for parallel manipulators \cite{Goss, Opt, pol1, pol2, Pol3, gough old, NR1, NR2, NR3, NR4,NR5, NR6, quat, tarokh, multi opt, parallel haptic, ACC iran, fuzzy NR, JP merlet}. These approaches can be categorized as analytical and numerical methods. In analytical methods, the problem is reduced to solving a polynomial functions involving multiple sine and cosine products and seeking closed-form solutions \cite{Goss,pol1,pol2,Pol3, gough old}. In most cases, a post-processing step is required to identify the feasible answer among various solutions of the polynomial. However, solving a high-order polynomial may be impractical and inefficient for real-time applications.

On the other hand, the Newton-Raphson method is often employed in numerical approaches to solve the problem \cite{NR1, NR2, NR3, NR4,NR5, NR6, quat, tarokh, multi opt, parallel haptic, ACC iran, fuzzy NR, JP merlet}. This class of algorithms involves computing the Jacobian matrix of the manipulator and its inverse at each iteration. However, such methods require a significant computational resource which may be infeasible for real-time applications. Furthermore, the algorithm is sensitive to the initial value that may harm the convergence to the actual solution.

For a 3SPR architecture, passive degree-of-freedom (DoF) are a function of active DoFs and possess very small amplitudes in comparison to the dimensions of the manipulator. The corresponding equations include complex algebraic relations that need to be solved, adding to the problem's computational complexity. The present paper focuses on a 3SPR parallel manipulator with a three DoF architecture to develop a new numerical approach for solving the FK problem. The algorithm is computationally efficient as it neglects the passive DoFs at the cost of some estimation error. In the proposed method, instead of formulating the problem as a non-convex problem, FK is solved in three steps. First, it is shown that the additional error introduced by neglecting the passive DoFs is upper bounded by the maximum amplitude of the passive DoFs. Then, it is demonstrated that the orientation of the manipulator can be estimated using the translational movements of the manipulator along the $z$-axis, called heave. It is also shown that estimation errors in the orientation are functions of the heave estimation error and the maximum amplitude of the passive DoFs. Next, using some results from the inverse kinematics (IK) problem, which is much easier to solve, a cost function is proposed that indirectly measures the distance between the actual and estimated FK values. Finally, an algorithm is developed to solve for FK, and its convergence to the actual value of heave is analytically investigated.

The rest of the paper is organized as follows. In Section II, the kinematics of a 3SPR parallel manipulator is reviewed. The proposed methodology and main results are presented in Section III. In Section IV, the theoretical results are validated by simulation. The paper is concluded in Section V. 
\section{Kinematics Review of 3SPR Manipulator}
In this section, the architecture of the parallel manipulator is elaborated. Then, inverse and forward kinematics problems are discussed.
\subsection{Manipulator Architecture}
Throughout the paper, $\mathbb{R}$, $\mathbb{R}^+$  and $\mathbb{N}$ refer to the sets of real numbers, positive real numbers and natural numbers, respectively.  Given any $ n \in \mathbb{N}$, $\mathbb{N}_n$ denotes the finite set $\{1,\ldots,n\}$. For any vector $v\in \mathbb{R}^{3\times 1}$, the expression of the vector in frame $\{F\}$ is denoted by $\leftidx{^F}{v}$. Let also $\leftidx{^F}{v}^x$, $\leftidx{^F}{v}^y$, and $\leftidx{^F}{v}^z$ denote the elements of $v$ in $x$, $y$ and $z$ directions of frame $\{F\}$. Fig.~\ref{3SPR} depicts the generic architecture of a 3SPR parallel manipulator. As seen in the figure, the base of the manipulator is a triangle with vertices $A_1, A_2$ and $A_3$. An inertial frame $\{I\}$ is attached to the base of the manipulator. The moving platform is defined by a triangle $B_1B_2B_3$, which is the same size as triangle $A_1A_2A_3$. A moving coordinate frame $\{M\}$ is also attached to the moving platform. 

The three vertices $A_1, A_2$ and $A_3$ are fixed at the point of attachment of spherical joints to the ground. Similarly, vertices $B_1, B_2$ and $B_3$ are the revolute joints attached to the moving platform. The origin $O'$ of the moving platform is obtained by finding the intersection of the orthogonals to the axes of the revolute joints. The inertial frame origin $O$ is obtained analogously with respect to the triangle $A_1A_2A_3$. Each spherical joint on the inertial frame is attached to the revolute joints on the moving frame by a prismatic joint. Let $\leftidx{^I}{a_i}\in \mathbb{R}^{3\times 1}$ and $\leftidx{^I}{b_i}\in \mathbb{R}^{3\times 1}$ denote the vector of position of $A_i$  and $B_i$, $i\in \mathbb{N}_3$, with respect to the inertial frame, respectively. Also, let $a_1 = [d_1,0,0]^\intercal$, $a_2 = [-d_2,d_3,0]^\intercal$ and $a_3 = [-d_2,-d_3,0]^\intercal$, where $d_1, d_2, d_3 \in \mathbb{R}^+$. Denote $l_1$, $l_2$ and $l_3$ as the length of each prismatic actuator and $l = [l_1,l_2,l_3]^\intercal\in \mathbb{R}^{3\times 1}$ as the vector of prismatic actuators' lengths.   

The parallel 3SPR architecture provides three degrees of freedom (DoF) for the manipulator. Denote $Z$, $\alpha$ and $\beta$ as the translational motion along the $z$ axis (heave), rotational motion around the $x$ axis (roll), and the rotational motion around the $y$ axis (pitch), respectively. Also, denote $X$, $Y$ and $\gamma$ as the small translation of the moving platform along the $x$ and $y$ axes and the small rotation around the $z$ axis (yaw), respectively. The set of small motions $(X,Y,\gamma)$ are generally known as parasitic motions \cite{NR1}.

\subsection{Inverse Kinematics}
Denote $R = R_z(\gamma) R_y(\beta) R_x(\alpha)$ as the X-Y-Z rotation matrix of the moving platform with respect to the fixed frame, and $P = [X,Y,Z]^{\intercal}$ as the coordinate of the moving platfrom with respect to the fixed frame. The position of each revolute joint in $\{I\}$ can be expressed as
\begin{equation}\label{first_loop}
   \leftidx{^I}{b_i} = P + R \times\leftidx{^M}{b_i}\hspace{0.5cm} i\in \mathbb{N}_3 
\end{equation}
or equivalently
\begin{equation}\label{second_loop}
  \leftidx{^I}{b_i} = \leftidx{^I}{a_i} + l_i\hat{s}_i  \hspace{1cm} i\in \mathbb{N}_3
\end{equation}
where $\hat{s}_i$ denotes the unit vector of the $i$th prismatic joint expressed in $\{I\}$. It results from \eqref{first_loop} and \eqref{second_loop} that
\begin{equation}\label{IK}
    l_i =  |P + R \times\leftidx{^M}{b_i} - \leftidx{^I}{a_i}| \hspace{0.5cm} i\in \mathbb{N}_3
\end{equation}
In other words, if the space of prismatic values is denoted by $\theta = (l_1,l_2,l_3)$ and the workspace feasible values by $\Theta = (Z,\alpha,\beta)$, the inverse kinematics equation \eqref{IK} defines a mapping $\Phi$ from the workspace to joint space represented by
\begin{equation}\label{IK_exact}
   \Phi: \mathbb{R}^3 \rightarrow \mathbb{R}^3, \hspace{0.5cm}\theta = \Phi(\Theta).
\end{equation}
\begin{remark}
In a 3SPR manipulator, to solve the IK problem for exact joint lengths, the parasitic motions must be computed as a function of the configuration  $\Theta$.  
\end{remark}
\subsection{General Form of Parasitic Motions}\label{parasitic section}
Let $T$ denote the transformation matrix from the inertial frame to the moving frame such that
\begin{equation*}
    T = \Bigg[\begin{array}{c|c}
R & P \\
\hline
[0,0,0] & 1
\end{array}\Bigg]
\end{equation*}
Recall that the coordinates of the spherical joints with respect to $\{M\}$ can  be expressed by using $T^{-1}$ that leads to
\begin{equation}\label{transformation_eq}
\leftidx{^M}{a_i} = R^{\intercal} \times(-P + \leftidx{^I}{a_i})  
\end{equation}
Since the revolute joints constrain the relative motion of spherical joints with respect to $\{M\}$,  it always holds that
\begin{equation}\label{mech_constraint}
  \leftidx{^M}{a_1}^y = 0,\hspace{0.5cm}\leftidx{^M}{a_2}^y = m (\leftidx{^M}{a_2}^x) ,\hspace{0.5cm}\leftidx{^M}{a_3}^y = -m (\leftidx{^M}{a_3}^x)  
\end{equation}
where $m = \frac{\leftidx{^M}{b_2}^y}{\leftidx{^M}{b_2}^x} = - \frac{\leftidx{^M}{b_3}^y}{\leftidx{^M}{b_3}^x}$. From \eqref{transformation_eq} and \eqref{mech_constraint}, the closed-form solutions for these motion can be obtained. For the yaw angle $\gamma= \frac{(ab+b^2)\sin{\alpha}\sin{\beta}}{(ab+b^2)\cos{\alpha} + c^2\cos{\beta}}$. The expression for the $X$ and $Y$ are page-long algebraic equations which are provided in Appendix~I.
\begin{remark}
To obtain exact solutions for both the inverse and forward kinematics, one must account for the required extra computations imposed by considering parasitic motions. 
\end{remark}

Consider $d_1$ = 1150 mm, $d_2$ = 500 mm and $d_3$ = 390 mm that exemplify typical dimensions of a 3SPR parallel manipulator as a motion platform. Fig.~\ref{distXY} illustrates the distribution of the ratios of $\frac{X}{Z}$ and $\frac{Y}{Z}$ for this manipulator. It is observed that these amplitude ratios are relatively small in comparison to the heave for different configurations. Hence, one way to solve the FK may be to consider a simplified form of the kinematic equation at the cost of some error. Therefore, the following problem is proposed.
\begin{figure}
\centering
	\includegraphics[scale=0.25]{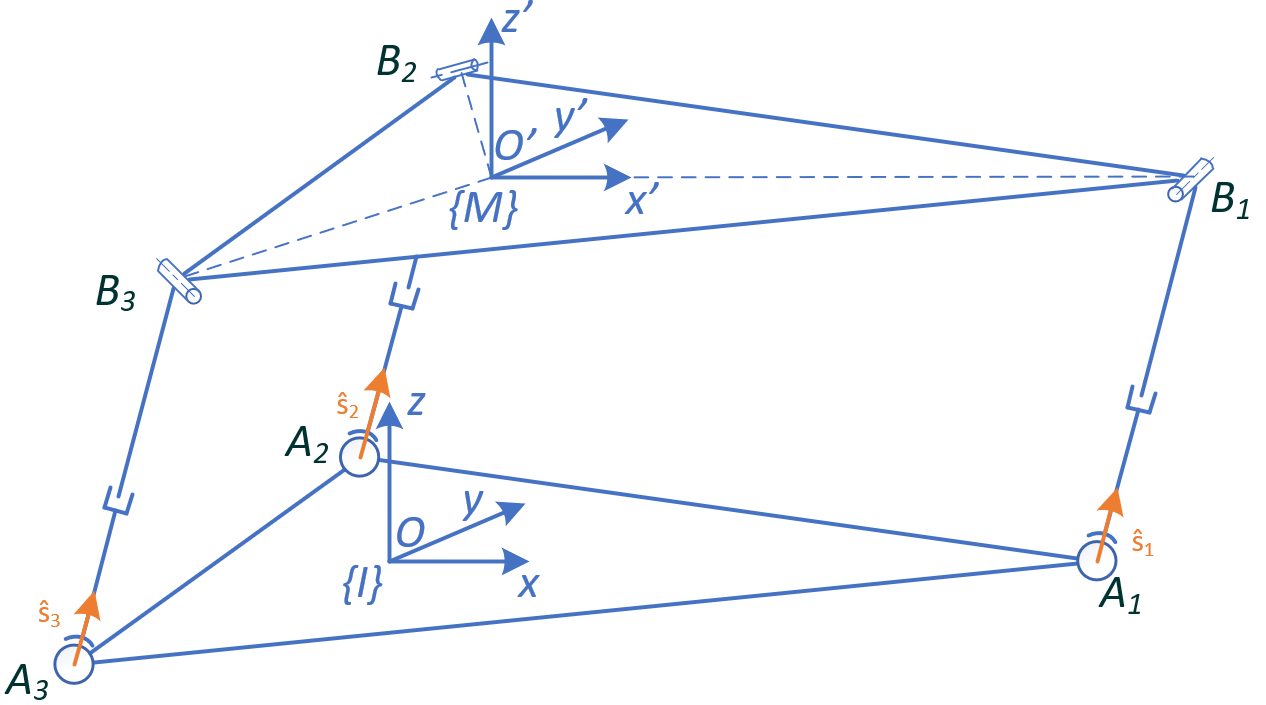}
	\caption{Architecture of a 3SPR manipulator}
	\label{3SPR}
\end{figure} 

\begin{problem}[Forward kinematics of a 3SPR manipulator]
For a given length of prismatic actuators $l_1, l_2$ and $l_3$, develop an algorithm to estimate the manipulator workspace configuration $(Z,\alpha,\beta)$. 
\end{problem}

\begin{figure}
\centering
	\includegraphics[scale=0.50]{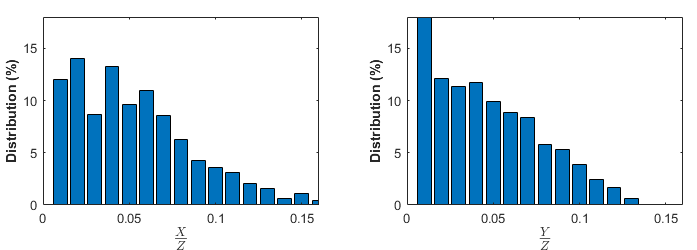}%
	\caption{Distribution of parasitic motion to heave ratios for a typical motion platform}
	\label{distXY}
\end{figure} 

\section{Main Results}
In this section, a method is proposed to solve Problem~1. 
\subsection{A Simplified Formula for IK}

Consider again the general form of the IK equation in \eqref{IK}. By neglecting the $X$ and $Y$ components of the motion, a simplified form of the IK may be introduced as 
\begin{equation}\label{IK_mapping_inexact}
   \tilde{ \theta} = \widetilde{\Phi}(\Theta)
\end{equation}
which has the following extended form
\begin{equation}\label{IK_New}
    \tilde{l_i} =  |\tilde{P} + R \times\leftidx{^M}{b_i} - \leftidx{^I}{a_i}| \hspace{0.5cm} i\in \mathbb{N}_3
\end{equation}
where $\tilde{l_i}$ denotes an inexact value of prismatic joint length for the $i$th limb, and $\tilde{P} = [0,0,Z]^{\intercal}$ is the simplified position of the manipulator. The following lemma establishes an upper bound on the error between the actual value $l_i$ and the inexact values of the joint lengths $\tilde{l_i}$. 
\begin{lemma}\label{lemma mu}
Let $\mu = \max\{|X|,|Y|\}$, then
\begin{equation*}
  {|l_i-\tilde{l_i}|}\leq\sqrt{2}\mu.  
\end{equation*}
\end{lemma}
\begin{proof}
The proof is provided in Appendix~II.A.
\end{proof}

\subsection{Roll \& Pitch Estimation from the Heave Estimates}
In this step, the objective is to estimate the roll and the pitch angle from the geometry of the manipulator, based on a given heave value. From \eqref{IK}, for $i=1$, it holds that
\begin{equation*}
    l_1^2 = (X - a + a\cos(\beta))^2 + (Z - a\sin(\beta))^2 + Y^2,
\end{equation*}
and with some simplifications, one has
\begin{equation*}
\begin{split}
    (2aX + 2a^2)\cos(\beta) + &2aZ\sin(\beta) = \\
    &2a^2 + (X^2+Y^2+Z^2) - l_1^2.
\end{split}
\end{equation*}
Solving for $\beta$ yields
\begin{equation*}
    {\beta} = \lambda \pm \omega,
\end{equation*}
where
\begin{equation}\label{pitch2}
\begin{split}
    \lambda &= \arcsin\Big(\frac{2a{Z}}{\sqrt{(2aX + 2a^2)^2 + (2a{Z})^2}}  \Big),\\
    \omega &= \arccos\Big(\frac{2a^2 + (X^2+Y^2+{Z}^2) - l_1^2}{\sqrt{(2aX + 2a^2)^2 + (2a{Z})^2}} \Big).
\end{split}
\end{equation}
By considering the zero pitch condition ($\beta=0$) for a zero heave ($Z =0$), the solution with the plus sign is ruled out and therefore
\begin{equation}\label{pitch_exact}
    {\beta} = \lambda - \omega,
\end{equation}
By neglecting the effects of movement in the $X$ and $Y$ directions, the pitch angle can be estimated as
\begin{equation}\label{pitch1}
    \hat{\beta} = \hat{\lambda} - \hat{\omega}
\end{equation}
where
\begin{equation}\label{pitch2}
\begin{split}
    \hat{\lambda} &= \arcsin\Big(\frac{2a{\hat Z}}{\sqrt{(2a^2)^2 + (2a{\hat Z})^2}}  \Big),\\
    \hat{\omega} &= \arccos\Big(\frac{2a^2 + (\hat{Z}^2) - l_1^2}{\sqrt{(2a^2)^2 + (2a{\hat Z})^2}} \Big),
\end{split}
\end{equation}
where $\hat{Z}$ denotes the estimated heave. 

To solve for the roll angle, consider the IK equation for the two rear limbs, i.e. \eqref{IK} for $i=2,3$, and subtract them to obtain
\begin{equation*}
\begin{split}
     l_2^2-l_3^2 &= -4c \Big(Y\cos(\alpha) - Y + Z\cos(\beta)\sin(\alpha)\\
     &+ X\sin(\alpha)\sin(\beta) + b\sin(\alpha)\sin(\beta)\Big),
\end{split}
\end{equation*}
which can be rewritten as
\begin{equation*}
\begin{split}
    Y\cos(\alpha) + \Big({Z}\cos({\beta}) &+ (X+ b)\sin({\beta})\Big)\sin(\alpha) \\
    &=\frac{l_3^2-l_2^2}{4c} + Y,
\end{split}
\end{equation*}
which yields 
\begin{equation}
  {\alpha} = \gamma \pm \kappa
\end{equation}
where
\begin{equation}\label{roll2}
\begin{split}
   \gamma &= \arcsin\Bigg(\frac{\Big({Z}\cos({\beta}) + (X+ b)\sin({\beta})\Big)}{\sqrt{Y^2 + \Big({Z}\cos(\hat{\beta}) + (X+ b)\sin({\beta})\Big)^2}} \Bigg)\\
   \kappa &= \arccos\Bigg(\frac{\frac{l_3^2-l_2^2}{4c} + Y}{\sqrt{Y^2 + \Big({Z}\cos({\beta}) + (X+ b)\sin({\beta})\Big)^2}} \Bigg).
\end{split}
\end{equation}
Noting  the zero roll condition ($\alpha=0$) for a zero heave ($Z =0$), the solution with the plus sign is ruled out and 
\begin{equation}\label{roll_exact}
  {\alpha} = \gamma - \kappa
\end{equation}
Neglecting the translational displacement along the $X$ and $Y$ axes, the roll angle can also be estimated as
\begin{equation}\label{roll2}
  \hat{\alpha} = \hat\gamma - \hat\kappa = \frac{\pi}{2} - \hat\kappa
\end{equation}
where 
\begin{equation}\label{roll3}
\hat\kappa = \arccos\Bigg(\frac{\frac{l_3^2-l_2^2}{4c}}{\sqrt{ \Big({\hat Z}\cos({\hat\beta}) + (b)\sin({\hat\beta})\Big)^2}} \Bigg).
\end{equation}
\begin{definition}
Define $\hat R$ as the estimated rotation matrix from \eqref{pitch1} and \eqref{roll2}. Let $e_z = Z-\hat{Z}$ denote the heave estimation error. From \eqref{pitch_exact} and \eqref{roll_exact}, let $\Psi_1$ and $\Psi_2$ represent the pitch and roll angles as functions of the heave and translational motion along the $X$ and $Y$ axes. Also, define $\rho = \max \lVert P-\tilde{P}\rVert$ and the region $\mathcal{D}=\{P|\lVert P-\tilde{P}\rVert\leq \rho\}$.
\end{definition}
\begin{lemma}\label{roll-pitch-estimate}
 Let $P\in\mathcal{D}$, it holds that 
\begin{equation}
    \begin{split}
        |\alpha - \hat \alpha| &\leq \mathcal{L}_1(\sqrt{2}\mu+e_z), \\
        |\beta -\hat{\beta}| &\leq \mathcal{L}_2(\sqrt{2}\mu+e_z),
    \end{split}
\end{equation}
where
\begin{equation*}
    \begin{split}
        \mathcal{L}_1 &= \max |\nabla \Psi_1|,\\
        \mathcal{L}_2 &= \max |\nabla \Psi_2|.
    \end{split}
\end{equation*}
\end{lemma}
\begin{proof}
The proof is provided in Appendix~II.B.
\end{proof}
\subsection{Heave Estimation by Optimization}
Consider the IK mapping in \eqref{IK_exact}. Let $\Theta$ and $\hat{\Theta}$ denote the actual and estimated workspace configurations, respectively. It follows that if $\hat{\Theta} = \Theta$, then $\hat{\theta} = \theta$. One way to formulate this observation is the following distance function 
\begin{equation*}
   \Lambda_1(\hat{\Theta}) =  (\theta - {\Phi}(\hat{\Theta}))^{\intercal}(\theta - {\Phi}(\hat{\Theta})),  
\end{equation*}
where it holds that $\Lambda_1(\hat{\Theta}) = 0$ if and only if $\hat{\Theta} = \Theta$. Considering the computational effort of $\Phi$, if one employs $ \widetilde{\Phi}$ from \eqref{IK_mapping_inexact} instead, the following cost function is constructed
\begin{equation}\label{cost_NotExt}
   \Lambda_2(\hat{\Theta}) =  (\theta - \widetilde{\Phi}(\hat{\Theta}))^{\intercal}(\theta -\widetilde{\Phi}(\hat{\Theta})).
\end{equation}
\begin{remark}\label{GD_exact}
Assume that for a given set of joint lengths $l$, the roll and pitch angles are estimated by employing \eqref{pitch1} and \eqref{roll2}. Then, $\Lambda_2$ can be parametrized as a single variable function of heave, denoted as $\Lambda_3(Z)$. Then, by moving along the gradient of this function with respect to $Z$, a sub-optimal value of the heave can be obtained.
\end{remark}
\begin{remark}\label{GD_app}
By estimating the values of the roll and pitch angles, an approximate direction of the gradient can be used to update the heave estimates as
\begin{equation}\label{z_update}
    \hat{Z}_{k+1} = \hat{Z}_{k} -\eta \widehat{\bigg(\frac{\partial\Lambda}{\partial Z}\bigg)},
\end{equation}
where $\eta$ is a proper update step size defined later.
\end{remark}
Considering Remarks \ref{GD_exact} and \ref{GD_app}, Algorithm~1 is proposed to solve for the FK problem. 
\begin{algorithm}[htp]
\caption{Forward kinematics estimation}\label{Alg-Decap}

\begin{algorithmic}[1]
  \Input{$\theta = (l_1,l_2,l_3)$, step size $\eta$ and $N$ iterations }
  \Output{Forward kinematics estimate $\hat{\Theta} = (\hat{Z},\hat{\alpha},\hat{\beta})$}
  \State Initialize iteration counter $k\leftarrow 0$
  \State Initialize $\hat{Z}_{k} \leftarrow \hat{Z}_0$
  \State Calculate $\hat{\beta}_k$ from \eqref{pitch1}
  \State Calculate $\hat{\alpha}_k$ from \eqref{roll2}
  \State Calculate $\hat{Z}_{k+1}$ from \eqref{z_update}
  \While {$k<N$}
  \State $\hat{Z}_\text{k} \leftarrow \hat{Z}_{k+1}$
  \State Calculate $\hat{\beta}_k$ from \eqref{pitch1}
  \State Calculate $\hat{\alpha}_k$ from \eqref{roll2}
  \State Calculate $\hat{Z}_{k+1}$ from \eqref{z_update}
  \State $k \leftarrow k + 1$
\EndWhile
\end{algorithmic}

\end{algorithm}

\begin{lemma}\label{ubound_gradL}
It always holds that
\begin{equation}\label{gradZBound}
    \Big|\frac{\partial \tilde{l}_i}{\partial Z}\Big|\leq 1,\hspace{0.5cm}i\in \mathbb{N}_3.
\end{equation}
\end{lemma}
\begin{proof}
The proof is provided in Appendix~II.C.
\end{proof}

\begin{lemma}\label{l1 inq}
The following inequality holds
\begin{equation*}
    \Big|\frac{\partial \Lambda_3}{\partial Z}\Big| \leq \lVert\tilde{l} - l\rVert_1
\end{equation*}
\end{lemma}
\begin{proof}
The proof is provided in Appendix~II.D.
\end{proof}
\begin{remark}
Recall that the discrepancy between $\tilde{l_i}$ and $l_i$, $i\in\mathbb{N}_3$, is due to the parasitic motions $X$ and $Y$.  In addition, parasitic motions affect the discrepancy between the estimated and exact rotation matrices $\hat R-R$. Let ${\epsilon_1}_i$, ${\epsilon_2}_i$, ${\epsilon_3}_i$ and ${\epsilon_4}_i$ denote the scaled compound effect of these errors, mathematically expressed by
\begin{equation*}
\begin{split}
    {\epsilon_1}_i &= \frac{[0\hspace{0.25cm}0\hspace{0.25cm}1]\Big([X\hspace{0.25cm}Y\hspace{0.25cm}0]+(\hat R-R)a_i\Big)}{\tilde{l_i}+l_i},\\ 
    {\epsilon_2}_i &= \frac{a_i^\intercal(\hat R-I)^\intercal(\hat R-I)a_i - a_i^\intercal(R-I)^\intercal(R-I)a_i}{\tilde{l_i}+l_i},\\
    {\epsilon_3}_i &= \frac{2[0\hspace{0.25cm}0\hspace{0.25cm}Z](\hat R-R)a_i}{\tilde{l_i}+l_i},\\
    {\epsilon_4}_i &= \frac{{X}^2+{Y}^2+2[X\hspace{0.25cm}Y\hspace{0.25cm}0](\hat R-R)a_i}{\tilde{l_i}+l_i}.
\end{split}
\end{equation*}
\end{remark}
The next lemma concerns a general form of the partial derivative of the cost function with respect to the optimal heave $Z$. 
\begin{lemma}\label{lem partial der}
The partial derivative of the cost function with respect to the heave can be expressed as
\begin{equation}\label{grad_exp}
   \frac{\partial \Lambda_3}{\partial Z} = \sum_{i=1}^3((1+\epsilon_{1i})(\hat Z-Z) + \epsilon_{2i}+\epsilon_{3i}+\epsilon_{4i})\frac{\partial \tilde{l}_i}{\partial Z}. 
\end{equation}
\end{lemma}
\begin{proof}
The proof is provided in Appendix~II.E.
\end{proof}
\begin{definition}
For a given step size $\eta$ define 
\begin{equation*}
   \delta = (1-\eta \sum_{i=1}^3(1+\epsilon_{1i})\frac{\partial \tilde{l}_i}{\partial Z}).
\end{equation*}
Also, define 
\begin{equation*}
    c_1 = \sum_{i=1}^3\lvert{1+\epsilon_1}_i\rvert, \hspace{1cm} c_2 = \sum_{i=1}^3\lvert{\epsilon_2}_i\rvert+\lvert{\epsilon_3}_i\rvert+\lvert{\epsilon_4}_i\rvert.
\end{equation*}
\end{definition}

\begin{theorem}\label{mainTh}
Let ${Z}$ and $\hat{Z}_0$ denote the actual heave and its initial estimate, respectively. Using Algorithm 1 with the step size $\eta <\frac{1}{c_1}$, then the heave estimate error after $N$ iterations is upper-bounded as described below
\begin{equation*}
\big|\hat{Z}_{N} - Z\big|\leq \max\big\{\big|\frac{c_2}{c_1}\big|\big\} + \delta^{N-1}\big|\hat{Z}_{0} - Z\big|.
\end{equation*}
\end{theorem}
\begin{proof}
From \eqref{z_update} and \eqref{grad_exp}, it holds that
\begin{equation}
\begin{split}
    \hat{Z}_{N} - Z &=  (\hat{Z}_{N-1} - Z) \\
    &-\eta \sum_{i=1}^3((1+\epsilon_{1i})(\hat{Z}_{N-1}-Z) + \epsilon_{2i}+\epsilon_{3i}+\epsilon_{4i})\frac{\partial \tilde{l}_i}{\partial Z}\\
    &\leq\big|\delta(\hat{Z}_{N-1} - Z)\big| + \eta c_2,
\end{split}
\end{equation}
 where the last inequality results from \eqref{gradZBound} and the triangle inequality. If the step size satisfies inequality $\eta <\frac{1}{c_1}$, then $\delta <1$. Let the latter inequality be expanded as
\begin{equation}
  \begin{split}
    \big|\hat{Z}_{N} - Z\big|&\leq \delta\big|\hat{Z}_{N-1} - Z\big| + \big|\eta c_2\big|\\
    &\leq \delta^2\big|\hat{Z}_{N-2} - Z\big| + (1+\delta)\big|\eta c_2\big|\\
    &\hspace{0.25cm}\vdots\\
    &\leq \delta^{N-1}\big|\hat{Z}_{0} - Z\big| + \sum_{i=0}^{N-1} \delta^i\big|\eta c_2\big|\\
    &\leq \delta^{N-1}\big|\hat{Z}_{0} - Z\big| + \sum_{i=0}^{\infty} \delta^i\big|\eta c_2\big|\\
    &=\delta^{N-1}\big|\hat{Z}_{0} - Z\big| + \frac{\big|\eta c_2\big|}{1-\delta}\\
    &= \delta^{N-1}\big|\hat{Z}_{0} - Z\big| + \big|\frac{c_2}{c_1}\big|.\\
\end{split}  
\end{equation}
This completes the proof.
\end{proof}
\begin{remark}
If the manipulator's motion is sufficiently smooth such that $|{\epsilon_1}_i|\ll1$, $|{\epsilon_2}_i|\ll1$, $|{\epsilon_3}_i|\ll1$ and $|{\epsilon_4}_i|\ll1$, then Theorem~\ref{mainTh} implies that the heave can be estimated sufficiently accurately with the maximum residual error of $\max\big\{\big|\frac{c_2}{c_1}\big|\big\}$.  
\end{remark}
\begin{remark}
The proposed method may be be generalized to other parallel manipulators if two necessary conditions are met: i. The ratio of the parasitic motion to the independent DoFs be small enough to simplify the kinematic equations with negligible error, and, ii. For a given set of joint lengths (angles), all independent DoFs can be formulated as a function of only one DoF.     
\end{remark}
\subsection{Computational Complexity of the Forward Kinematics}\label{compJ}
In this part, we compare the computational complexity of the JB method with that of the proposed method. For the IK mapping in \eqref{IK_exact}, a first-order approximation at the $n$th time instant may be written as
\begin{equation*}
    \theta_n = \Phi(\Theta_n)=\Phi(\Theta_{n-1}+\Delta\Theta)\approx\Phi(\Theta_{n-1})+\frac{\partial\Phi}{\partial\Theta}\Delta\Theta,
\end{equation*}
where $J = \frac{\partial\Phi}{\partial\Theta}$ is known as the Jacobian in IK problem. Also, recall that $\Phi(\Theta_{n-1}) = \theta_{n-1}$, therefore
\begin{equation}\label{JEq}
    \theta_n - \theta_{n-1} = \Delta \theta   \approx J\Delta\Theta.
\end{equation}
Hence, to solve FK, by starting from an initial configuration, the work space variables can be approximated by first computing the elements of the Jacobian $J$ and then solving \eqref{JEq} for $\Delta\Theta$. 

To solve FK with the proposed method, at each iteration, estimates of the pitch \eqref{pitch1} and the roll angle \eqref{roll2} should be computed. Then, heave estimate must be updated using \eqref{z_update}. For the current 3SPR manipulator, Table~I reports the required number of elementary arithmetic operations for both methods (these values are obtained by using  MATLAB\textsuperscript\textregistered's built-in function named  socFunctionAnalyzer). It is observed that the required arithmetic operations for the JB method is approximately thirty times larger than that for one iteration of the proposed method. The relatively large algebraic expression of the parasitic motions explain the observed discrepancy between required computation of the JB and the proposed method. 

\begin{table}[t] 
\centering
\caption{Required number of elementary arithmetic operations for the JB and the proposed methods}
    \begin{tabular}{| c | c | c |c| }
    \hline
     & $\pm$ & $\times /$ & Total  \\ \hline
    JB method & 1953 & 11087 &13040\\ \hline
    Proposed method (One Iteration) & 205 & 199& 404  \\ 
    \hline
    \end{tabular}
    
\end{table}

\section{Simulations}
We run simulations for a 3SPR parallel manipulator with dimensions described in \ref{parasitic section}, where the allowable ranges for DoF parameters are $Z\in [0,100]$ mm, $\alpha\in[-3.5^{\circ},3.5^{\circ}]$ and $\beta\in[-1.5^{\circ},1.5^{\circ}]$. To assess the performance of the proposed algorithm under different motion conditions, we consider a combined parabolic, ramp and sinusoidal trajectory. Let $u(t-\tau)$ denote the unit step function delayed by $\tau$ and assume that the heave, pitch and roll angles are time-varying functions described by $Z= (50 + 1.6t^2)(u(t)-u(t-2.5)) + (85-10t)(u(t-2.5)-u(t-5)) + 15 \sin(\frac{\pi}{2} t -3\pi)$, $\alpha= 0.4t(u(t)-u(t-5)) + 2 \cos(0.4\pi t)u(t-5)$ and $\beta= (-0.4t+1)(1-u(t-2.5))+ \sin(2\pi f_{\text{pitch}} t)u(t-5)$ for 20 seconds. At each time instant, the estimated heave of each method from the previous instant is used as the initial value for the current instant. For the JB method, we run simulations by taking into account the full information of the parasitic motions. For the proposed method, we select a step size of $\eta = 0.08$. Figs.~\ref{heave_er} and \ref{high_er} compare the estimated DoF parameters for six iterations of the proposed method with that of the JB method using $f_{\text{pitch}}$ = 0.2~Hz and $f_{\text{pitch}}$ = 1~Hz, respectively. It is observed that for the ramp and parabolic portions of the trajectory, both methods track the generated motion with satisfactory precision. For the sinusoidal portion of the trajectory with the frequency $f_{\text{pitch}}$ = 0.2~Hz in the pitch channel, on the other hand the JB method estimates the DoF parameters with negligible error as observed in Fig.~\ref{heave_er}. However, Fig.~\ref{high_er} shows that for a frequency as high as $f_{\text{pitch}}$ = 1~Hz, the estimates obtained by the JB method exhibit considerable error while those obtained by the proposed method display much smaller errors. This larger estimation error is partly contributed by the first-order approximation of the inverse kinematics function in the JB method. Note that the first-order approximation error is larger for rapid motions. Finally, we investigate the effect of larger iterations on the estimation error of the proposed method. Fig.~\ref{error comp} illustrates the estimation error of the proposed method for the same trajectory for six and thirty iterations. It is observed that for the larger numbers of iterations the estimation error is smaller, as expected.  
\begin{figure}
\centering
	\includegraphics[scale=0.48]{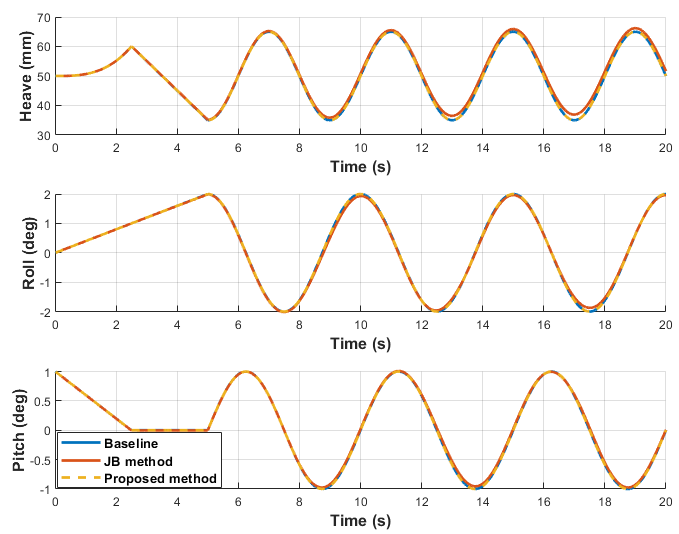}
	\caption{Estimation of the FK variables by employing the proposed algorithm with six iterations and the JB method with $f_{\text{pitch}}$ = 0.2~Hz}
	\label{heave_er}
\end{figure} 

\begin{figure}
\centering
	\includegraphics[scale=0.48]{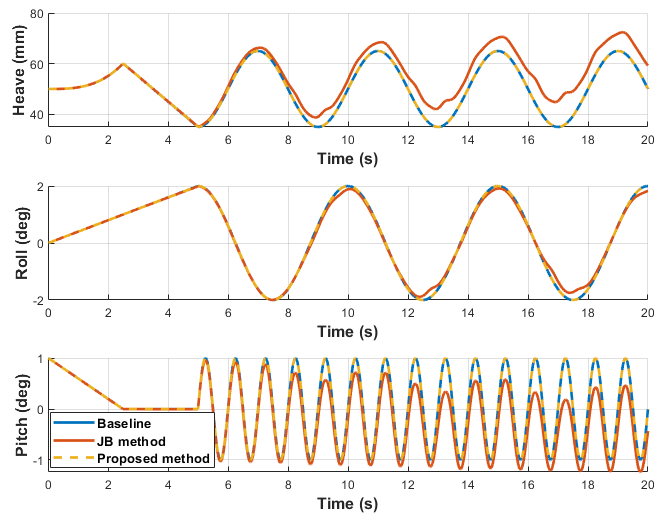}
	\caption{Estimation of the FK variables by employing the proposed algorithm with six iterations and the JB method with $f_{\text{pitch}}$ = 1~Hz}
	\label{high_er}
\end{figure} 

\begin{figure}
\centering
	\includegraphics[scale=0.43]{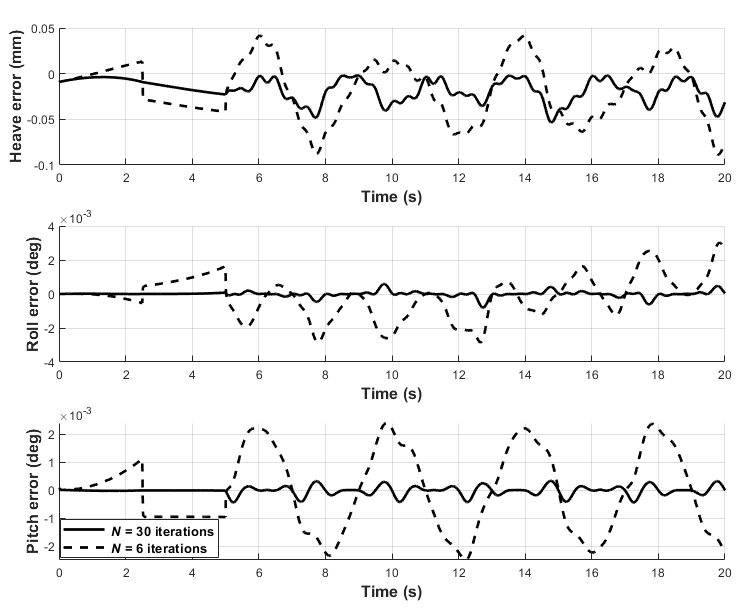}
	\caption{Comparison of the estimation error for six and thirty iterations of the proposed algorithm}
	\label{error comp}
\end{figure} 

\section{Conclusions}
In this work, the FK estimation of a 3SPR parallel manipulator was investigated and a novel numerical method was developed accordingly. The proposed algorithm is computationally more efficient compared to Jacobian-based approaches. Furthermore, analytical results were established to evaluate the performance of the proposed algorithm in terms of accuracy. Simulation results support the theoretical findings. As a future research direction, one can extend the algorithm to a more general class of parallel manipulators.  
\section*{ACKNOWLEDGMENT}

This research was supported by MITACS and Touché  Technologies Inc. under grant FR54600. We would also like to thank the reviewers for their constructive comments.


\appendices
\section{Algebraic Expression of  Parasitic Motions}
The parasitic motion along the $x$ axis is a fractional expression
$X=\frac{P_x}{Q_x}$ where
\begin{equation*}
    \begin{split}
        P_x &= d_2^5\cos^4\alpha + d_1^2d_2^3\cos^4\alpha - d_2^5\cos^3\alpha\cos\beta \\
        &+ d_2^5\sin^4\alpha\sin^4\beta + 2d_1d_2^4\cos^4\alpha\\ &- 2d_2^3d_3^2\cos^2\alpha\cos^2\beta+ d_1^2d_2^3\sin^4\alpha\sin^4\beta\\
        &-2_1d_2^4\cos^3\alpha\cos\beta - d_2d_3^4\cos\alpha\cos^3\beta\\
        &+ Z d_2^4\cos^3\alpha\sin\beta + 2d_2^5\cos^2\alpha\sin^2\alpha\sin^2\beta\\
        &- d_1^2d_2^3\cos^3\alpha\cos\beta + d_2d_3^4\cos^2\alpha\cos^2\beta\\
        &+ 2d_2^3d_3^2\cos^3\alpha\cos\beta + 2d_1d_2^4\sin^4\alpha\sin^4\beta\\
        &+ 4d_1d_2^4\cos^2\alpha\sin^2\alpha\sin^2\beta\\
        &- d_2^5\cos\alpha\cos\beta\sin^2\alpha\sin^2\beta\\
        &+ 2d_1d_2^2d_3^2\cos^3\alpha\cos\beta + Z d_1^2d_2^2\cos^3\alpha\sin\beta\\
        &+ Z d_3^4\cos\alpha\cos^2\beta\sin\beta + 2d_1^2d_2^3\cos^2\alpha\sin^2\alpha\sin^2\beta\\
        &- d_2^3d_3^2\cos^2\beta\sin^2\alpha\sin^2\beta- 2d_1d_2^2d_3^2\cos^2\alpha\cos^2\beta \\
        &+(d_2^4+d_1d_2^3) Z \cos\alpha\sin^2\alpha\sin^3\beta\\
        &+ 2Z d_1d_2^3\cos^3\alpha\sin\beta\\
        &+ Z d_1^2d_2^2\cos\alpha\sin^2\alpha\sin^3\beta + Z d_2^2d_3^2\cos\beta\sin^2\alpha\sin^3\beta\\
        &+ Z d_2^2d_3^2\cos^3\beta\sin^2\alpha\sin\beta+ Z d_2^4\cos\alpha\cos^2\beta\sin^2\alpha\sin\beta\\
        &- 2d_1d_2^4\cos\alpha\cos\beta\sin^2\alpha\sin^2\beta\\
        &- 2d_1d_2^2d_3^2\cos^2\beta\sin^2\alpha\sin^2\beta - d_1^2d_2d_3^2\cos^2\beta\sin^2\alpha\sin^2\beta\\
        &+(2d_2^3d_3^2- d_1^2d_2^3)\cos\alpha\cos\beta\sin^2\alpha\sin^2\beta\\ 
        &+ 2Z d_2^2d_3^2\cos^2\alpha\cos\beta\sin\beta\\ 
        &+ Z d_1^2d_2^2\cos\alpha\cos^2\beta\sin^2\alpha\sin\beta\\
        &+ 2d_1d_2^2d_3^2\cos\alpha\cos\beta\sin^2\alpha\sin^2\beta\\
        &+ Z d_1d_2d_3^2\cos\beta\sin^2\alpha\sin^3\beta 
        + Z d_1d_2d_3^2\cos^3\beta\sin^2\alpha\sin\beta\\
        &+ 2Z d_1d_2^3\cos\alpha\cos^2\beta\sin^2\alpha\sin\beta\\
        &+ 2Z d_1d_2d_3^2\cos^2\alpha\cos\beta\sin\beta,
    \end{split}
\end{equation*}
and
\begin{equation*}
    \begin{split}
        Q_x &= \cos\alpha\cos\beta \times\\
        \Big(&d_2^4\cos^2\alpha+ d_3^4\cos^2\beta+ d_1^2d_2^2\cos^2\alpha + d_2^4\sin^2\alpha\sin^2\beta\\
        &+ 2d_1d_2^3\cos^2\alpha + d_1^2d_2^2\sin^2\alpha\sin^2\beta+ 2d_2^2d_3^2\cos\alpha\cos\beta\\
        &+ 2d_1d_2^3\sin^2\alpha\sin^2\beta+ 2d_1d_2d_3^2\cos\alpha\cos\beta\Big).
    \end{split}
\end{equation*}
The parasitic motion along the $y$ axis is a fractional expression $Y=\frac{P_y}{Q_y}$ where
\begin{equation*}
    \begin{split}
        P_y&= d_2^3d_3^2\sin^3\alpha\sin^3\beta + Z d_3^4\cos^3\beta\sin\alpha\\ 
        &+ Z d_2^4\cos^2\alpha\cos\beta\sin\alpha  + Z d_3^4\cos\beta\sin\alpha\sin^2\beta\\
        &- d_1d_3^4\cos^2\beta\sin\alpha\sin\beta - d_2d_3^4\cos^2\beta\sin\alpha\sin\beta\\
        &+ d_1d_2^2d_3^2\sin^3\alpha\sin^3\beta + d_2^3d_3^2\cos^2\alpha\sin\alpha\sin\beta\\
        &+ Z d_2^2d_3^2\cos\alpha\sin\alpha\sin^2\beta + d_1d_2^2d_3^2\cos^2\alpha\sin\alpha\sin\beta\\
        &- d_2^3d_3^2\cos\alpha\cos\beta\sin\alpha\sin\beta + 2Z d_1d_2^3\cos^2\alpha\cos\beta\sin\alpha\\
        &+ d_2d_3^4\cos\alpha\cos\beta\sin\alpha\sin\beta + Z d_1^2d_2^2\cos^2\alpha\cos\beta\sin\alpha\\
        &+ 2Z d_2^2d_3^2\cos\alpha\cos^2\beta\sin\alpha + Z d_1d_2d_3^2\cos\alpha\sin\alpha\sin^2\beta\\
        &- 2d_1d_2^2d_3^2\cos\alpha\cos\beta\sin\alpha\sin\beta - d_1^2d_2d_3^2\cos\alpha\cos\beta\sin\alpha\sin\beta\\
        &+ 2Z d_1d_2d_3^2\cos\alpha\cos^2\beta\sin\alpha,
    \end{split}
\end{equation*}
and
\begin{equation*}
    \begin{split}
        Q_y&= \cos\alpha\times\\
        \Big(&d_2^4\cos^2\alpha+ d_3^4\cos^2\beta + d_1^2d_2^2\cos^2\alpha+ d_2^4\sin^2\alpha\sin^2\beta\\
        &+ 2d_1d_2^3\cos^2\alpha + d_1^2d_2^2\sin^2\alpha\sin^2\beta + 2d_2^2d_3^2\cos\alpha\cos\beta\\
        &+ 2d_1d_2^3\sin^2\alpha\sin^2\beta + 2d_1d_2d_3^2\cos\alpha\cos\beta\Big).
    \end{split}
\end{equation*}

\section{Proof of Lemmas}
\subsection{Proof of Lemma~\ref{lemma mu}}
From \eqref{IK} and \eqref{IK_New} and the triangle inequality one has
\begin{equation}\label{l_uB}
     |l_i-\tilde{l_i}| \leq |P - \tilde{P}|=\sqrt{X^2+Y^2}\leq\sqrt{2\mu^2}=\sqrt{2}\mu.
\end{equation}
\subsection{Proof of Lemma~\ref{roll-pitch-estimate}}
Since $\Psi_1(X,Y,Z)$ is continuously differentiable in the bounded region $\mathcal{D}$, it has a bounded Lipschitz constant within that region denoted by $\mathcal{L}_1$, which yields
\begin{equation}
\begin{split}
    |\alpha - \hat \alpha| &\leq \mathcal{L}_1|(X,Y,Z)-(0,0,\hat{Z})|  \\
    &= \mathcal{L}_1\sqrt{{X}^2+{Y}^2 + (Z-\hat{Z})^2}\leq\mathcal{L}_1(\sqrt{2}\mu+e_z)
\end{split}
\end{equation}
where the last inequality results from the triangular inequality. Using the same line of argument, an upper bound for the roll estimate error can also be established.
\subsection{Proof of Lemma~\ref{ubound_gradL}}
It follows from \eqref{IK} that
\begin{equation}\label{l_Extended}
 \tilde{l}_i = \sqrt{(\tilde{P}+(R-I)a_i)^{\intercal}(\tilde{P}+(R-I)a_i)}   
\end{equation}
and it holds that
\begin{equation*}
  \frac{\partial \tilde{l}_i}{\partial Z} = \frac{[0\hspace{0.25cm}0\hspace{0.25cm}1](\tilde{P}+(R-I)a_i)}{\sqrt{(\tilde{P}+(R-I)a_i)^{\intercal}(\tilde{P}+(R-I)a_i)} }. 
\end{equation*}
Recall that for any vector $v$ it holds that 
\begin{equation*}
  \Big|\frac{[0\hspace{0.25cm}0\hspace{0.25cm}1]v}{\sqrt{v^{\intercal}v} }\Big|\leq 1
\end{equation*}
This completes the proof.
\subsection{Proof of Lemma~\ref{l1 inq}}
From \eqref{cost_NotExt} and Remark~\ref{GD_exact} the cost function is described as
\begin{equation}
   \Lambda_3({Z}) =  \frac{1}{2}\sum_{i = 1}^3 (l_i - \tilde{l}_i)^2,
\end{equation}
Therefore
\begin{equation}\label{lambda_Grad}
   \frac{\partial \Lambda_3}{\partial Z} =  \sum_{i = 1}^3 (l_i - \tilde{l}_i) \frac{\partial \tilde{l}_i}{\partial Z}.
\end{equation}
The proof follows from Lemma~\ref{ubound_gradL} and the definition of $L_1$ norm.

\subsection{Proof of Lemma~\ref{lem partial der}}
From \eqref{lambda_Grad}, one has
\begin{equation}
   \frac{\partial \Lambda_3}{\partial Z} =  \sum_{i = 1}^3 (\tilde{l}_i^2 - {l_i}^2)\frac{1}{l_i + \tilde{l}_i} \frac{\partial \tilde{l}_i}{\partial Z}.
\end{equation}
Using \eqref{IK} and \eqref{l_Extended}, the proof follows .
\end{document}